\documentclass[conference]{IEEEtran}
\IEEEoverridecommandlockouts

\usepackage{cite}
\usepackage{amsmath,amssymb,amsfonts,amsthm}
\usepackage{algorithmic}
\usepackage{graphicx}
\usepackage{textcomp}
\usepackage{xcolor}
\usepackage{enumitem}
\usepackage{hyperref}
\def\BibTeX{{\rm B\kern-.05em{\sc i\kern-.025em b}\kern-.08em
    T\kern-.1667em\lower.7ex\hbox{E}\kern-.125emX}}

\DeclareMathOperator*{\argmin}{argmin}
\newcommand\tdots{\hbox to 0.7em{.\hss.\hss.}\,}
\let\originalleft\left
\let\originalright\right
\renewcommand{\left}{\mathopen{}\mathclose\bgroup\originalleft}
\renewcommand{\right}{\aftergroup\egroup\originalright}
\let\originalmid\mid
\renewcommand{\mid}{\!\originalmid\!}

\newlist{researchquestions}{enumerate}{1}
\setlist[researchquestions]{label*=\textbf{RQ\arabic*}}

\newlist{experimentquestions}{enumerate}{1}
\setlist[experimentquestions]{label*=\textbf{Q\arabic*}}

\theoremstyle{plain}
\newtheorem{theorem}{Theorem}[section]
\newtheorem{proposition}[theorem]{Proposition}
\newtheorem{proposition*}[theorem]{Proposition}
\newtheorem{lemma}[theorem]{Lemma}

\theoremstyle{definition}

\newtheorem{example}{Example}
\theoremstyle{remark}

\usepackage{tikz}
\usepackage{textcomp}
\usepackage{hyperref}
\usepackage{lipsum}

\newcommand\copyrighttext{%
  \footnotesize \textcopyright 20XX IEEE. Personal use of this material is permitted.
  Permission from IEEE must be obtained for all other uses, in any current or future
  media, including reprinting/republishing this material for advertising or promotional
  purposes, creating new collective works, for resale or redistribution to servers or
  lists, or reuse of any copyrighted component of this work in other works.
}
\newcommand\copyrightnotice{%
\begin{tikzpicture}[remember picture,overlay]
\node[anchor=south,yshift=10pt] at (current page.south) {\fbox{\parbox{\dimexpr\textwidth-\fboxsep-\fboxrule\relax}{\copyrighttext}}};
\end{tikzpicture}%
}

\begin{document}
\thispagestyle{plain}  
\pagestyle{plain} 

\title{
Theoretical Evaluation of Asymmetric Shapley Values for Root-Cause Analysis
\thanks{\hspace{-8pt}Supported by the European Union project RRF-2.3.1-21-2022-00004 within the framework of the Artificial Intelligence National Laboratory Program.}
}

\newcommand*{\authorblind}[1]{{}}

\author{\IEEEauthorblockN{Domokos M. Kelen}
\IEEEauthorblockA{%
\textit{HUN-REN SZTAKI\IEEEauthorrefmark{1}\thanks{\hspace{-8pt} \IEEEauthorrefmark{1}Hungarian Research Network, Institute for Computer Science and Control}}\!,\\
Budapest, Hungary\\
kdomokos@sztaki.hu
}
\and
\IEEEauthorblockN{Mihály Petreczky}
\IEEEauthorblockA{%
\textit{CNRS, University Lille}\\
Lille, France \\
mpetrec@gmail.com}
\and
\IEEEauthorblockN{Péter Kersch}
\IEEEauthorblockA{%
\textit{Ericsson Hungary}\\
Budapest, Hungary \\
peter.kersch@ericsson.com}
\and
\IEEEauthorblockN{András A. Benczúr}
\IEEEauthorblockA{%
\textit{HUN-REN SZTAKI\IEEEauthorrefmark{1}}\!,\\
Budapest, Hungary\\
benczur@sztaki.hu
}
}

\maketitle
\copyrightnotice

\begin{abstract}
In this work, we examine Asymmetric Shapley Values (ASV), a variant of the popular SHAP additive local explanation method. ASV proposes a way to improve model explanations incorporating known causal relations between variables, and is also considered as a way to test for unfair discrimination in model predictions.
Unexplored in previous literature, relaxing symmetry in Shapley values can have counter-intuitive consequences for model explanation.
To better understand the method, we first show how local contributions correspond to global contributions of variance reduction. Using variance, we demonstrate multiple cases where ASV yields counter-intuitive attributions, arguably producing incorrect results for root-cause analysis.
Second, we identify generalized additive models (GAM) as a restricted class for which ASV exhibits desirable properties. We support our arguments by proving multiple theoretical results about the method. 
Finally, we demonstrate the use of asymmetric attributions on multiple real-world datasets, comparing the results with and without restricted model families using gradient boosting and deep learning models.
\end{abstract}


\section{Introduction}

Removal-based feature attribution methods have gained huge popularity in recent years.
To gain insight into machine learning models, they calculate expectations of model predictions with and without revealing certain feature values and examine the change in expectation. 
Methods like SHAP~\cite{lundberg2017unified} yield easy-to-interpret explanations for model predictions by attributing parts of the prediction to individual features.

In some cases, the prediction can be made from either of multiple predictors with overlapping information content, such as when multiple variables in a causal chain are present in the dataset.
The Asymmetric Shapley Values (ASV)~\cite{frye2020asymmetric} framework promises to solve this problem by relaxing the symmetry property in the original Shapley formula.
However, as we show in this paper, relaxing symmetry can yield unexpected results in case of complex interactions between variables. 

A method to handle overlapping information content
is essential in the case of root cause analysis, where the explanation method is used for identifying the underlying issues for a certain prediction value, with multiple complex causal relationships among different levels of observations. For example, in radio telecommunications, base station configuration affects load and interference, which then affect the radio channel quality and eventually the user experience. We can try to assert the cause of user experience degradation through local explanations using SHAP; however, the degradation can most likely be predicted from either of multiple variables in the causal chain. 
In this case, we prefer to prioritize the root (or \emph{distal}) causes over more immediate ones in the explanation, and ASV provides us with a convenient way to do so.

In other cases, such as detecting unfair discrimination, the same framework can be used to instead prefer the immediate (or \emph{proximate}) causes. In social sciences, family attributes affect education, which determines income and poverty, which eventually predict crime rate. One might prefer the explanations to attribute shared contribution to resolving variables that mediate the effects of sensitive attributes~\cite{frye2020asymmetric}.

The previous two use cases primarily consider the effect of shared contributions, where multiple predictors provide redundant information about a target variable through correlation or other nonlinear means. However, interacting variables can also result in additional information about the target variable that cannot be inferred from either of the predictor variables alone but is revealed by the variables when observed together. Such interactions, which we call \emph{complex interactions}, can lead to arguably incorrect results in the case of distal attribution schemes and disproportionate contributions in the case of proximate attribution schemes. It is important to consider the consequences of different kinds of interaction in various use cases, as the attributions produced by the method can be misleading without a proper understanding of the different ways that variables can interact.

In this work, we conduct a theoretical analysis of ASV, focusing on regression, and formalizing the effects of relaxing symmetry in terms of variance reduction of the target variable. We seek answers to the following research questions:
\begin{researchquestions}[itemindent=1em]
    \item  Can we reason about the attributions of ASV formally?\label{rq1}
    \item Is ASV a reliable method for root-cause analysis? \label{rq2}
    \item Can we avoid undesired behavior by using GAMs?  \label{rq3}
\end{researchquestions}
In our analysis, we use the framework of variance reduction to show how ASV can yield counter-intuitive results in the presence of nonlinearity or even simple non-additive variable interactions.
As a solution, we propose applying ASV in conjunction with generalized additive models (GAMs). In Theorem~\ref{st:redundancy}, our main result proves a sufficient condition for correct behavior when ASV is applied to GAMs.

Our paper is structured as follows. In Section~\ref{sec:related}, we summarize related literature. In Section~\ref{sec:background}, we recall the definitions of SHAP and ASV. In Section~\ref{sec:varred}, we describe how the behavior of local explanations can be studied through the framework of variance reduction. In Section~\ref{sec:asymmetric}, we characterize the behavior of ASV, and present theoretical examples where ASV yields counter-intuitive results. In Section~\ref{sec:gam}, we show that many of the listed counter-intuitive behaviors vanish when restricting the usage of ASV to GAMs. In Section~\ref{sec:classification}, we discuss generalization to classification. Finally, in Section~\ref{sec:experiments}, we demonstrate our findings on real-world datasets. 

\section{Related literature}\label{sec:related}

SHAP, introduced in~\cite{lundberg2017unified}, has sparked a surge of research in recent years into related explainability methods. In this paper, we study Asymmetric Shapley Values~\cite{frye2020asymmetric},
which provides a way to account for known causal relationships in the dataset when calculating contributions.
We will enumerate a number of other works that build upon SHAP to account for causality in attributions; however, we are not aware of works on formally assessing the suitability of ASV for root-cause analysis.

Results on measuring causal contributions~\cite{heskes2020causal,wang2021shapley,jung2022measuring,li2022explanatory} enforce restricting explanation to causal effects by relying on Pearl's do-calculus~\cite{pearl2012calculus}.
In contrast, our pathological examples (Section~\ref{sec:asymmetric}) stem from the combination of asymmetry and nonlinearity, which are orthogonal to the question of direct and indirect causal effects. 
Do-calculus-based methods combine well with ours by detecting confounders while our results detect complex interactions.
Shapley flow~\cite{wang2021shapley}
extends the do-calculus idea by assigning contributions to causal edges instead of variables.
Causal Shapley Values~\cite{heskes2020causal} decompose the causal effects into direct and indirect effects, which can then be combined both with symmetric SHAP and ASV~\cite[Fig.~1]{heskes2020causal}, hence can leverage on our results.
Finally, the main contribution in~\cite{jung2022measuring} consists of criteria to efficiently infer do-calculus values from fixed measurement data, which is non-trivial in our use cases with no intervention opportunity when training the model.

In a completely different approach to model explanation, in~\cite{agarwal2021neural} the use of generalized additive models (GAM)~\cite{hastie2017generalized} is proposed to combine the expressivity of a general model with the inherent intelligibility of GAMs but without providing local post-hoc explanations.  
The relationship of SHAP and GAMs are studied in~\cite{bordt2023shapley}, which introduces $n$-Shapley Values to explain $n$-wise symmetric interactions.
As a step beyond, in Section~\ref{sec:gam} we formalize the behavior ASV when applied in conjunction with GAMs vs.~arbitrary models.

We study ASV by explaining the variance reduction of the target variable in the SHAP framework. Our variance-reduction-based approach is equivalent in expectation to SAGE~\cite{covert2020understanding,frye2021shapley} with the MSE loss function. In general, analysis of variance (ANOVA) and variance-based sensitivity analysis are popular approaches~\cite{owen2014sobol, HUANG199849, herren2022statistical, da2021kernel}, which we use here to study ASV. Note that our ASV variance reduction explanation does not modify the validity of the other SHAP axioms~\cite{lundberg2017unified} beyond symmetry.

\section{Background}\label{sec:background}
\subsection{Shapley Additive Explanations (SHAP)}
In~\cite{lundberg2017unified}, SHAP is defined as a way to explain a specific prediction of a machine learning model by assigning contribution values to different input features of the model. More specifically, given model $f$, feature set $\{X_s \mid s\in\mathbb{S}\}$, and a specific point $x$ from the dataset, the contribution $\phi_j$ of feature $j\in \mathbb{S}$ can be defined as
\begin{gather}\label{eq:original_shap}
    \phi_j(f, x)=\sum_{S\subseteq \mathbb{S}\setminus\{j\}}\frac{|S|!\left(|\mathbb{S}|-|S|-1\right)!}{|\mathbb{S}|!}\;\varphi_{j}(f, x, S)\textrm{,}\\
    \label{eq:valuefunc}
     \textrm{where\;}\varphi_j(f, x, S)=v_{f(x)}(S\cup \{j\})-v_{f(x)}(S)\textrm{.}
\end{gather}
Here $v_{f(x)}(S)$ is called the \emph{value function} and represents the function of evaluating the model $f$ at point $x$ while only using features from the coalition $S$. The evaluation is usually done by taking the expectation of $f(x)$ conditional on the in-coalition feature values $S$:
\begin{gather}
    \label{eq:valufunc_expectation}
    v_{f(x)}(S) = E\left[f(X)\mid X_s=x_s, s\in S\right]\textrm{.}
\end{gather}
With $\phi_0  = E(f(X))$, the feature contributions
sum up to the prediction value:
\begin{equation}\label{eq:shap_additive}
    \sum_{j=0}^{|\mathbb{S}|} \phi_j (f,x) = f(x).
\end{equation}

We note that two commonly used variants of SHAP are conditional
and marginal SHAP
~\cite{model_or_data}. Similar to~\cite{frye2020asymmetric}, we focus on \emph{conditional} SHAP, where the dependencies between features are accounted for in the expected value in~\eqref{eq:valufunc_expectation}.
\subsection{Asymmetric Shapley Values}
When the prediction power of two variables together is not equal to the sum of their individual prediction powers, then these variables are said to interact. The original weighting scheme of Shapley values, described in~\eqref{eq:original_shap}, guarantees such interactions to be distributed equally between variables. In \cite{frye2020asymmetric}, ASV is introduced as a way to prefer certain variables over others by assigning interactions asymmetrically. This can be useful for example in the presence of causal relationships, where one might prefer assigning shared contribution to either root causes or immediate causes.
Formally, Asymmetric Shapley Values are defined as
\begin{align}\label{eq:expectation_weighted}
    \phi_j^\omega(f(x))=\sum_{\pi\in\Pi}\omega(\pi)\,\varphi_j(f, x, R_j)\textrm{,}
\end{align}
where $\Pi$ is the set of all possible permutations of the model features, ${R_j=\{i: \pi(i)<\pi(j)\}}$, $\omega$ is a weighting over permutations, and $\varphi_j(f, x, R_i)$ is defined in~\eqref{eq:valuefunc}.
The definition allows any weighting function, for example, the distal (or root cause) function~\cite{frye2020asymmetric} 
\begin{equation}\label{eq:distal}
    \omega_{distal}(\pi)\propto
    \begin{cases}
    1&\parbox[m]{.55\columnwidth}{if $\pi(i)<\pi(j)$ for all $i,j$ where $i$ is a causal ancestor of $j$}\vspace{0.4em}\\
    0&\textrm{otherwise,}
    \end{cases}
\end{equation}
which essentially means that causal ancestors are always revealed before variables affected by them.
The intuition given behind this approach is exemplified by assuming two identical $X_1, X_2$, however with $X_1$ being a causal ancestor of $X_2$. 
In this case, one might want to attribute all importance to $X_1$ instead of sharing it evenly.

The difference between ASV and SHAP is how variable interactions are handled. Two features $i$ and $j$ are said to interact when
\begin{equation}\label{eq:shap_interaction}
    v_{f(x)}(\{i\})+v_{f(x)}(\{j\})\neq v_{f(x)}(\{i,j\}),
\end{equation}
i.e. the effect of revealing the value of both at the same time to the model can not be additively predicted from revealing them separately. This can be the case, e.g., when they contain redundant information. In such cases, the difference between the left and right-hand side of~\eqref{eq:shap_interaction} needs to be added to $\phi_i(f, x)$ and $\phi_j(f, x)$ such that~\eqref{eq:shap_additive} holds. SHAP distributes the interaction value uniformly between the interacting variables, while ASV assigns it to the variable revealed last in each specific permutation $\pi$ of \eqref{eq:expectation_weighted}.

\section{Variance Reduction in the SHAP framework}\label{sec:varred}
The correctness of explanation methods is hard to quantify or benchmark, as there is rarely any ground truth to compare against. In the case of SHAP and ASV, even general behavior is hard to reason about, as these methods explain model predictions locally, i.e., each $f(x)$ prediction separately in the input point~$x$. A practical approach is to study them based on aggregate behavior, e.g., in~\cite{frye2020asymmetric}, the authors observe the behavior of ASV through the expectation of $\phi_j(f, x)$ over $x\sim X$. However, it is important to pay attention to how well the average reflects local behavior, as contribution values can have both positive and negative values.

In our case, the primary goal is to evaluate ASV for the regression task. Optimizing for mean squared error results in the model $f(x)$ approximating the conditional expectation of the target $T$, i.e., $f(X)\approx E[T\mid X]$, also implying
\begin{gather}\label{eq:model_cond_exp}
    E[f(X)\mid X_S]\approx E[E[T\mid X]\mid X_S]= E[T\mid X_S]
\end{gather}
when using conditional SHAP,
where $X_S = \{X_s\mid s\in S\}$.

As a result, however, using the average of local contribution values is pointless, as it is guaranteed to approximate zero. To see this, notice first that $\phi_0=E[f(X)]$ as in~\eqref{eq:shap_additive}, and second that the expectation $E[\phi_j(f, X)]=0$ for all $j\neq 0$, since for any $S$, $f$, and $x$, due to the law of total expectation,
\begin{gather}
    E[v_{f(X)}(S)]=E\left[E\left[f(X)\mid  X_S\right] \right]=E[f(X)]\textrm{,}
\end{gather}
implying $E[\varphi_j(f,x,S)]=0$ in turn for any $j, S$ in \eqref{eq:valuefunc}.

As a solution, some authors\cite{man2021best,scavuzzo2022feature} as well as the SHAP Python package consider features with large absolute Shapley values important. The average of the absolute contribution values $|\phi|$ is also used as feature importance in~\cite{lundberg2018consistent,molnar2018guide}.
However, the average absolute contribution value does not have a theoretic foundation, which again makes it hard to formally reason about the behavior of the method.

Instead, to answer~\ref{rq1}, we propose observing the changes in the variance of the target variable $T$ with different feature sets revealed to the model. We first show that this value is strongly tied to the local behavior of the model. Let us define a new value function $w$ to replace $v$ of~\eqref{eq:valuefunc}:
\begin{gather}\label{eq:w_value_func}
    w_{f(x),t}(S) \stackrel{def}{=} (t-v_{f(x)}(S))^2.
\end{gather}
On the one hand, the behavior of $w$ directly corresponds to the behavior of $v$, as it is just a simple transformation of the latter. On the other hand, averaging $w_{f(x),t}(S)$ over the dataset for a given $S$ approximates the variance of $T$ conditioned on features from $S$, called the residual sum of squares:
\begin{gather}
    E[w_{f(X),T}(S)]=E[(T-v_{f(X)}(S))^2]\\
    =E\left[(T-E(T \mid X_S))^2\right]\label{eq:condvar}
\end{gather}
where $X_S=\{X_s\}_{s \in S}$. The value in \eqref{eq:condvar} is also called conditional variance~\cite{spanos2019probability}. Notice that
\begin{gather}
\phi_0=E[w_{f(X),T}(\emptyset)]=E\left[(T-E(T))^2\right]=\sigma^2(T),
\end{gather}
i.e., the average contribution of the empty set $\phi_0$ becomes the variance itself, and that
\begin{gather}
R^2\cdot \sigma^2(T)=\sigma^2(T)-\sigma^2(T-E(T\mid X_S))=\label{eq:r2_1}\\
=E[w_{f(X),T}(\emptyset)]-E[w_{f(X),T}(S)],\label{eq:r2_2}
\end{gather}
where $R^2$ is the \emph{coefficient of determination} from statistics.

Since \eqref{eq:original_shap} and \eqref{eq:valuefunc} are linear transformations of the value function, all of the above means that when SHAP is used to explain the value function $w$ of \eqref{eq:w_value_func} locally, then the averages of contributions explain the variance reduction of the target variable globally in the SHAP framework. The contribution of the empty set, i.e., $\phi_0$, is equal to the unconditional variance, while the addition of predictor features lowers this variance, such that in the end the sum of the contributions equals ${E\left[(T-E(T \mid X_S))^2\right]}$. The resulting contributions are equivalent in expectation to the global explanations provided by SAGE~\cite{covert2020understanding} for the squared error loss function.

In what follows, we denote $R^2$ unnormalized as
\begin{equation}\label{eq:linear_info}
    L_{\scriptscriptstyle T}(X) \stackrel{def}{=} \sigma^2(T) - \sigma^2(T-E[T\mid X]),
\end{equation}
also allowing multiple variables in place of $X$, denoted as, e.g., $L_{\scriptscriptstyle T}(X,Y)$. The value of $L_{\scriptscriptstyle T}(X)$ can be interpreted as \emph{the predictive power of $X$} when no other variables are present. Note that a positive amount of variance reduction is realized as a \emph{negative contribution} of the variable in the SHAP framework: when averaged, summing $\phi_0$ to the contribution values equals the reduced variance of the target.

We also propose a definition to characterize the relation of the variance reduction of two attributes $X$ and $Y$ separately as opposed to the pair $X,Y$ together, similar to the classic Shapley interaction value~\eqref{eq:shap_interaction}. We define
\begin{equation}
    W_{\scriptscriptstyle T}(X;Y) \stackrel{def}{=} L_{\scriptscriptstyle T}(X,Y) - L_{\scriptscriptstyle T}(X) - L_{\scriptscriptstyle T}(Y) ,\label{eq:linear_interaction_info}
\end{equation}
the \emph{interaction of variance reduction}. The value $L_{\scriptscriptstyle T}(X)$ is guaranteed to be positive, while $W_{\scriptscriptstyle T}(X;Y)$ can be either positive or negative.

\section{Analysis of ASV through variance reduction}\label{sec:asymmetric}
Using the notation introduced in Section~\ref{sec:varred}, we can now answer~\ref{rq1}. We can express the contributions assigned by ASV in terms of variance reduction as follows.
\begin{proposition}\label{prop:asv}
    For a given permutation $\pi$, using $w$ of~\eqref{eq:w_value_func} in place of $v$ in \eqref{eq:valuefunc}, and ${R_j=\{i: \pi(i)<\pi(j)\}}$,
    \begin{gather}
        -E[\varphi_j(f, X, R_j)] 
        = L_{\scriptscriptstyle T}(X_{\pi(j)})+W_{\scriptscriptstyle T}(X_{\pi(j)}; R_j).
    \end{gather}
\end{proposition}
\begin{proof}
    From the definitions in~\eqref{eq:valuefunc},~\eqref{eq:w_value_func},~\eqref{eq:linear_info},~\eqref{eq:linear_interaction_info} and observing \eqref{eq:r2_1} and \eqref{eq:r2_2}, the proposition immediately follows.
\end{proof}

To illustrate the meaning of Proposition~\ref{prop:asv}, recall the distal weighting function $\omega$ of~\eqref{eq:distal}. Along with~\eqref{eq:expectation_weighted}, it filters out any $\pi$ in which the causal ordering is flipped for any two variables. For the remaining $\pi$, variable $X_j$ gets contribution negatively proportional to the variance of $T$ that it explains, plus the interaction value $W$ between it and its predecessors. The total contribution of $X_j$ is then an average over every such $\pi$.

Next, we show three examples where ASV can be argued to produce counter-intuitive results. ASV reveals variables to the model in the order specified by the permutation $\pi$ in order to prefer the variables that are revealed earlier in the explanation. This approach seems intuitive, as variables earlier in the permutation get their full contribution, while later variables are assigned leftover prediction power.
However, in some cases, the intuition fails. In the three examples below, we reveal two variables in the order $(X_1, X_2)$, and show that ASV assigns disproportionate contributions to $X_2$.

\begin{example}[Pairwise independence]\label{ex:pairwise}
Let $X_1$ and $X_2$ be pairwise independent of $T$, however with $T=f(X_1, X_2)$ for some $f$. With the permutation $(X_1, X_2)$, $\textrm{ASV}$ assigns all contribution to $X_2$ for predicting $T$.
\end{example}
\noindent A pair of variables can be pairwise independent, but at the same time not mutually independent together with the target variable.

\begin{example}[Nonlinearity]  \label{ex:nonlinearity}
For $X_1,X_2\sim N(0,1)$ independent Gaussians, let
    $T=(2X_1 + 2X_2)^2$\textrm{.}
Both variables contribute the same amount, yet with the permutation $(X_1, X_2)$, $\textrm{ASV}$ assigns three times the importance of $X_1$ to $X_2$ when measured using variance reduction. For an exact calculation of contributions, see Appendix~\ref{apx:example_2}.
\end{example}

\begin{example}[Non-additive effects]\label{ex:nonadditive} Let
$X_1$ and $X_2$ be as in Example~\ref{ex:nonlinearity} and $T = X_1\cdot X_2$.
Here $E[T\mid X_1]=0$, therefore with the permutation $(X_1, X_2)$, $\textrm{ASV}$ assigns all contribution to $X_2$ for predicting $T$, even though both variables contributed equally to $T$.
\end{example}

In all three of these examples, the main source of counter-intuitive behavior is that $X_2$ gets assigned more contribution \emph{because} it is introduced after $X_1$, more contribution than if it was the only predictor used, which goes directly against the intuition of the ordering of the predictors in $\pi$. In the examples, the effects are not the result of causal relationships between $X_1$ and $X_2$, rather they are artifacts of relaxing symmetry itself. However, real-world relationships can contain a combination of multiple effects, including causal and nonlinear ones. As an arbitrary causal example, we could have, e.g., $X_2=X_1+2Y$ with $Y\sim N(0,1)$ and $T=(X_1+X_2)^2$, resulting in the same incorrect contributions as in Example~\ref{ex:nonlinearity}.

Thus, answering~\ref{rq2}, we conclude that no, ASV is not always a reliable way of assigning contributions in root-cause analysis, as the effects demonstrated by Examples~\ref{ex:pairwise},~\ref{ex:nonlinearity}, and~\ref{ex:nonadditive} can result in proximate instead of root causes getting a large share of the contributions. The counter-intuitive behavior is due to the fact that the inequality
\begin{equation}\label{eq:fail_leq}
    W_{\scriptscriptstyle T}(X,Y)=L_{\scriptscriptstyle T}(X,Y)-L_{\scriptscriptstyle T}(X)-L_{\scriptscriptstyle T}(Y)\stackrel{?}{\leq} 0
\end{equation}
\emph{does not} hold in general: the prediction power of using two variables at the same time can be greater than the sum of their prediction powers separately, since the interaction between variables can be quite complex. Thus we can have two kinds of interaction between variables: interactions that result in the combined prediction power being less, and interactions that result in the combined prediction power being more than the sum of the individual prediction powers of the variables. We call the former kind of interaction \emph{redundant information} and the latter \emph{complex interaction}. The actual measured $W_T$ value is the sum of interactions of different kinds.

\section{ASV on Generalized Additive Models}\label{sec:gam}
From an interpretability standpoint, complex interactions are undesirable: they cannot be attributed cleanly to any of the involved variables. Rather, complex interactions are inherently the property of the variables being used together.
However,
attributing redundant information in causal settings is straightforward both in theory and in practice, as described and demonstrated by ASV. In this section, we propose using generalized additive models~\cite{hastie2017generalized} (GAMs) in an attempt to exclude undesirable complex interactions. We study the behavior of ASV when applied to GAMs, and prove that ASV exhibits certain desirable properties when applied on GAMs.

GAMs are a special class of prediction model, which can be written as a sum of smooth functions of different features:
\begin{equation}\label{eq:GAM_def}
    g(E[T\mid X])\approx f_0 + f_1(X_1)+\tdots+f_n(X_n),
\end{equation}
where $g$ is called the link function. Observe that $f_0$ is nonessential, as its effects can be assimilated within $g$ or~$f_i$. The explainability properties of GAMs have been studied before~\cite{agarwal2021neural,lou2012intelligible}, as well as various ways of approximating them using machine learning models. While these restricted models are strictly less powerful than their unrestricted counterparts, sacrificing some prediction power for interpretability can be reasonable in many domains and applications where the trustworthiness of the results is critical~\cite{agarwal2021neural}.

Note that GAMs themselves are additive, for which marginal SHAP~\cite{lundberg2017unified} guarantees that contributions reflect the outputs of individual $f_i$ functions. The same is, however, \emph{not} true for conditional SHAP~\cite{model_or_data}, where~\eqref{eq:model_cond_exp} remains true. Illustrating this point, assume two identical predictor variables. As also observed in~\cite{bordt2023shapley}, a GAM is free to choose either of them for predicting $T$. Therefore, hiding one of the variables could result in the model losing information when using marginal SHAP.
In contrast,
in conditional SHAP, the expectation is forced to account for either of them being present
through their joint distribution.

To reason about ASV and GAMs formally, we propose the following definitions. We define the \emph{restricted} conditional expectation of $T$ given $X_1,\tdots,X_n$ as
\begin{equation}\label{eq:restricted_conditional_expectation}
    E^r[T\mid X_1, \tdots, X_n]\stackrel{def}{=}\argmin_{\{f_i\}} \sigma^2\!\left(T-\sum_i f_i(X_i)\right),
\end{equation}
where $f_i$ are measurable and $E[f_i(X_i)^2]< \infty$.
Equation~\eqref{eq:restricted_conditional_expectation} gives the minimum variance predictor of $T$ that can be written as the sum of functions of $X_i$, i.e., as a GAM. We use the notation $E^r$ as a contrast to regular conditional expectation, which is equivalent to using a single measurable function over all $X_i$ to minimize variance~\cite{brockwell2009time}.
The function $E^r$ can be interpreted as the best possible GAM while using the identity link function. Notice that with only one predictor variable, the definitions of functions $E$ and $E^r$ are equivalent. For a discussion of the existence of $E^r$, see Appendix~\ref{apx:existence}.

Further, from the definition of~\eqref{eq:restricted_conditional_expectation}, we can derive matching definitions for restricted versions of $L_T$ and $W_T$:
\begin{align}\label{eq:restricted_linear_info}
    L_{\scriptscriptstyle T}^r(X) &\stackrel{def}{=} \sigma^2(T) - \sigma^2(T-E^r[T\mid X])\\
    W_{\scriptscriptstyle T}^r(X;Y) &\stackrel{def}{=} L_{\scriptscriptstyle T}^r(X,Y) - L_{\scriptscriptstyle T}^r(X) - L_{\scriptscriptstyle T}^r(Y) .\label{eq:restricted_linear_interaction_info}
\end{align}
Next, we prove that ASV exhibits a number of desirable properties when applied in combination with GAMs. Our first theorem refers back to Example~\ref{ex:pairwise}.
\begin{theorem}[Additivity of contributions]\label{st:cond_indep}
    If $X,Y$ are independent, then
        $L^r_T(X) + L^r_T(Y)=L^r_T(X,Y)$.
\end{theorem}
\noindent See Appendix~\ref{apx:cond_indep} for proof. Theorem~\ref{st:cond_indep} proves that under the GAM restriction, the predictive power of two independent variables together always equals the sum of their individual predictive powers. It also resolves Example~\ref{ex:pairwise}, as $X_1$ and $X_2$ must get the same contribution regardless of order of inclusion. Since in the example both predictor variables are also pairwise independent of $T$, both variables get a contribution score of zero: the full prediction power of the model is the result of a complex interaction, which is excluded from $E^r$, as desired.

Observe that Theorem~\ref{st:cond_indep} also applies to Examples~\ref{ex:nonlinearity} and~\ref{ex:nonadditive}. Specifically, in Example~\ref{ex:nonlinearity}, both $X_1$ and $X_2$ can be used to predict $T$ only to a certain degree, with half of its variance remaining, however, their contributions this time are equal. See Appendix~\ref{apx:example_2_e_r} for the exact calculation.
In Example~\ref{ex:nonadditive}, $E^r[T\mid X_1,X_2]=0$, therefore, similar to Example~\ref{ex:pairwise}, both  variables get zero contribution.

Finally, we turn to the question raised by~\eqref{eq:fail_leq}. With the following theorem, we prove that if the prediction function $F$ can be decomposed as $F_X(X)+F_Y(Y)$ (i.e., is a GAM), then ASV works as intended in most cases, shifting the contribution values to prefer variables earlier in the permutation $\pi$.
\begin{theorem}[Upper bound of $W^r_{\scriptscriptstyle T}$]\label{st:redundancy}
 \begin{gather*}
    W^r_{\scriptscriptstyle T}(X_1,\tdots,X_n;Y_1,\tdots,Y_m)\leq -2cov\left(F_{\scriptscriptstyle X}(X), F_{\scriptscriptstyle Y}(Y)\right)
 \end{gather*}
 assuming $\exists F_X, F_Y, cov(F_X(X), F_Y(Y))<\infty$ such that
 \begin{gather}
     F_{\scriptscriptstyle X}(X)+F_{\scriptscriptstyle Y}(Y)=E^r[T\mid X_1,\tdots,X_n,Y_1,\tdots,Y_m].
 \end{gather}
\end{theorem}
\noindent For a proof of the theorem, see Appendix~\ref{apx:leq_proof}.

Theorem~\ref{st:redundancy} means that as long as the covariance of $F_X(X)$ and $F_Y(Y)$ is positive, the interaction term $W$ is negative, i.e.,~\eqref{eq:fail_leq} is essentially true. Unfortunately, the theorem comes with the caveat that when the covariance is negative, the interaction \emph{can} be positive, as illustrated by the next example.
\begin{example}[Rank deficiency]\label{ex:deficient}
    Let $A,B\sim N(0,1)$ be independent joint Gaussian variables, and let
    \begin{gather}
        X_1=(0.1A+B)\textrm{, \,}
        X_2=(0.1A-B)\textrm{, and }
        T=A\textrm{.}
    \end{gather}
    Then $E[T\mid X_1, X_2]=5X_1+5X_2=T$.
\end{example}
Example~\ref{ex:deficient} is similar to the ones listed in Section~\ref{sec:asymmetric} in that the target $T$ cannot be efficiently predicted from $X_1$ or $X_2$ alone, but it can be predicted from $X_1$ and $X_2$ together. However, Example~\ref{ex:deficient}  is special in the sense that it persists even when using predictors according to $E^r$, i.e., the relations are all additive. In fact, the example remains valid even if we restrict the model family to linear regression. However, the example is quite unnatural, as the problematic behavior occurs because both variables $X_1$ and $X_2$ contain the same dominant noise variable, which can only be canceled out when observing both together. Such problems can be avoided on a theoretical level by assuming that noise terms are independent, but in practice, this might not always hold.

Answering \ref{rq3}, the GAM restriction filters out a large class of undesired complex interactions including those of Examples~\ref{ex:pairwise},~\ref{ex:nonlinearity}, and~\ref{ex:nonadditive}; however certain hard-to-interpret interactions can happen even with the restriction. Fortunately, as we see in Section~\ref{sec:experiments}, such relationships are uncommon in practice.

\section{Extension to classification}\label{sec:classification}
As stated before, our primary goal is to study ASV for the regression task. Aside from mutually exclusive events, probabilities are inherently non-additive,
implying the same for the classification task. Practical modeling approaches usually involve having the model approximate the \emph{logits} or log-odds of the probabilities, i.e., $\log \frac{p}{1-p}$, and applying the logistic sigmoid, its inverse, to the model output. Thus explaining the probabilities themselves using ASV may result in having to deal with nonlinear effects similar to Example~\ref{ex:nonlinearity}. Similarly, using GAMs for modeling probabilities directly is unlikely to result in accurate models.
At the same time,
logits are often treated as additive, e.g., in logistic regression. Therefore, we propose observing the output of the model before applying the sigmoid function, i.e., explaining the logits, as also done in~\cite{agarwal2021neural}. Using GAMs to predict logits is straightforward, with the link function $g$ of~\eqref{eq:GAM_def} chosen as the log-odds function.

Generalizing our results to classification seems a nontrivial task at first glance, as theorems~\ref{st:cond_indep} and~\ref{st:redundancy} observe the behavior of ASV through variance. Moreover, the target variable in classification is the ground-truth probabilities, which are unobservable in general. However, notice the fact that the only actual assumption made in our analysis is that the model approximates a conditional expectation in~\eqref{eq:model_cond_exp}. Further, since $f(X)=E[f(X)\mid X]$ in general, the model function can always be interpreted as a conditional expectation. Thus our approach to generalizing our results is by explaining the variance of the output of the full model when predicted using subsets of features, i.e., we set $T=f(X)$.

Ultimately, we consider the non-additivity of probabilities a limitation inherent to classification when being explained using additive explanation methods such as SHAP and ASV. However, explaining logits instead is a natural solution whenever applicable.

\section{Approximation in practice}\label{sec:rias}
In this Section, we describe how we approximate the proposed definitions in practice. Regression modeling tasks, when using the mean squared error objective function, end up approximating the conditional expectation of the target variable $E$, as described in Section~\ref{sec:varred}. To approximate $E^r$ instead, we need to train GAMs,
which we can achieve by using restricted versions of model families.

Gradient-boosted decision tree (GBDT) models~\cite{lightgbm} train an ensemble of decision tree models, where the final prediction is the sum of the predictions made by the individual trees. The training is done iteratively, always approximating the residual error. To achieve the GAM restriction, each individual tree needs to use variables of a single feature or feature group $X_i$. This way, the end result can be written as a sum of sets of trees, each depending on the value of a single $X_i$.

The function $E^r$ can be approximated using other model families as well. For example with neural networks, the restriction can be achieved by modifying the architecture of the network to represent a sum of predictors based on different features or feature groups~\cite{agarwal2021neural}. We primarily use GBDT-based models because of the tabular nature of our datasets. We run experiments implemented by LightGBM~\cite{lightgbm}, where the necessary restriction can be met using the \emph{interaction\_constraints} flag. We include further experiments using neural network models~\cite{agarwal2021neural} in our source code repository.

To evaluate ASV with conditional expectations, we follow the approach also used in
\cite{lipovetsky2001analysis}, which means training a separate model for each coalition $S$ to predict $E[T\mid X_S]$, and using these models as an approximation of $v_{f(x)}(S)$ for a given $x$ with different feature subsets.
In some experiments below, we also report contribution attributable to complex interactions, denoted by $\phi_{\mathcal{I }}$. The size $\phi_{\mathcal{I}}$ is determined by how much more powerful an unrestricted model is than its restricted counterpart.
Further analysis could be applied to attribute parts of $\phi_{\mathcal{I}}$ to different features, however, this is out of scope for this work.

\section{Experiments}\label{sec:experiments}
In this section, we demonstrate the effect of applying ASV to GAMs on real-world datasets.
Our main goals in this section are thus to see:
\begin{experimentquestions}
    \item how common complex interactions are in practice, both between variables and groups of variables;\label{eq1}
    \item whether Example~\ref{ex:deficient} manifests in practice;\label{eq2}
    \item the extent to which the predictive power of models can be attributed to complex interactions, by comparing the performance of GAMs to that of unrestricted models.\label{eq3}
\end{experimentquestions}
We make the code of our experiments publicly available\footnote{\url{https://github.com/proto-n/shap-asv-icdm}}.
\subsection{Datasets and methodology}
    We use the following datasets from the UCI Machine Learning Repository~\cite{uci_repo}. Other than unifying the target variable in the second dataset (PM2.5), no additional pre-processing is done on the UCI datasets. 
    \vspace{-\parskip}
    \begin{description}[leftmargin=1em, style=sameline]
    \item[Communities and Crime Unnormalized (CaCU).] Combines socio-economic data from the '90 Census, law enforcement data from the 1990 Law Enforcement Management and Admin Stats survey, and crime data from the 1995 FBI UCR. Semantic groups are listed in Appendix~\ref{apx:data}.\vspace{0.2em}
    \item[PM2.5 Data of Five Chinese Cities (PM2.5).] Contains air-quality data in Beijing, Shanghai, Guangzhou, Chengdu, and Shenyang with hour granularity, also including meteorological data for each city~\cite{chinese_cities}. As a unified target, we average all measurement stations into a single variable in all cities.\vspace{0.2em}
    \item[Superconductivity.] Contains features extracted from superconductors along with the critical temperature~\cite{hamidieh2018data}.\vspace{0.2em}
    \item[Productivity Prediction of Garment Employees (Garment).] Includes attributes of the garment manufacturing process and the productivity of the employees~\cite{Rahim_2021, Imran_2019}.\vspace{0.2em}
    \end{description}
    We also present  evaluation on a large proprietary dataset:
    \begin{description}[leftmargin=1em, style=sameline]
    \hyphenation{tel-co}
    \item[Mobile Telecommunications (Telco).] Real-world mobile telco. data, consisting of performance management (PM) data from radio access network cells with 15-minute granularity and configuration management (CM) data with daily granularity. We list feature groups in Appendix~\ref{apx:data}. Causality relations are displayed in Figure~\ref{fig:causality_dag}.
    \end{description}
    Relevant statistics of datasets are presented in Table~\ref{tab:datasets}. Train/validation/test splits are done in ratios of $0.8/0.1/0.1$, and evaluations are presented as measured on the test set. GBDT experiments are conducted using LightGBM~\cite{lightgbm}.

    \begin{table}[htbp]
      \caption{Datasets used.}.
      \label{tab:datasets}
      \centering
    \begin{tabular}{|l|r|r|l|}
    \hline
    \textbf{Name}&\textbf{Features} & \textbf{Rows} & \textbf{Target variable}\\
    \hline
    CaCU & 124 & 2215  & ViolentCrimesPerPop\\
    \hline
    PM2.5 & 14 & 262920  & PM\\
    \hline
    Superconductivity & 81 & 21263  & critical\_temp\\
    \hline
    Garment & 14 & 1197  & actual\_productivity \\
    \hline
    Telco & 78 & 19343000  & downlink\_throughput\\
    \hline
    \end{tabular}
    \end{table}

\subsection{Complex interaction example from the CaCU dataset}\label{sec:experiment_example}
    In this section, we demonstrate a real-world example of complex interactions on a concrete example from the CaCU dataset. Here the goal is to predict the relative number of violent crimes in communities (\emph{ViolentCrimesPerPop}) from 124 features related to income, education, age, etc.
    
    \begin{figure}[htbp]
    \begin{center}
    \centerline{\includegraphics[width=1\linewidth]{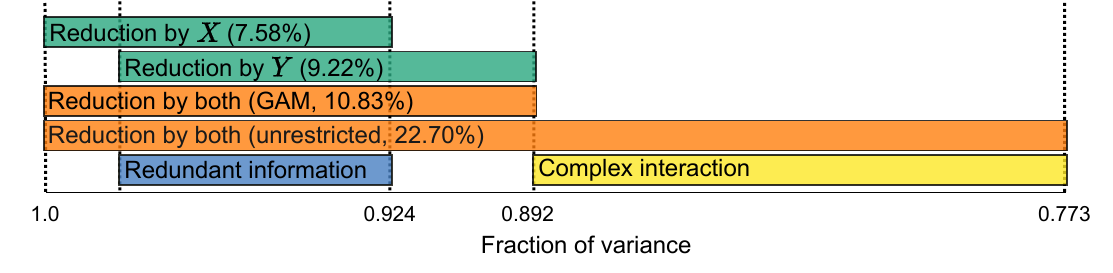}}
    \vspace{-0.5em}
    \caption{Different kinds of interactions in the example of Section~\ref{sec:experiment_example}.}\label{fig:experiment_example}
    \end{center}
    \vspace{-1.5em}
    \end{figure}
    
    To start, we take two highly correlated variables: the percentage of the population that is 12-29 in age $X$ (\emph{agePct12t29}) and the percentage of males who have never married $Y$ (\emph{MalePctNevMarr}). Either of them can be used to reduce the target variance by a considerable amount: $X$ reduces the variance by $7.58\%$, and $Y$ by $9.22\%$ when used on their own, see Figure~\ref{fig:experiment_example}.
    Because they are highly correlated ($r=0.79$), one could expect them to contain redundant information on the target. However, their combined predictive power is actually larger than the sum of their separate predictive powers, together they reduce the variance by $22.70\%$. Because of this, the ASV contribution of the second variable ends up being almost double its individual predictive power.
    
    When checking the same example using restricted models, we find that the combined predictive power of the two variables is much more in line with what we expect from two correlated variables, together they reduce the variance by $10.83\%$ under the GAM restriction. The contribution of the second variable is now reduced compared to its individual predictive power by a significant margin.

\subsection{Interactions between pairs of variables}
    
    Investigating \ref{eq1}, we measure the interactions between pairs of variables in the datasets to see the prevalence of complex interactions. In the CaCU dataset, we measure the interactions between possible pairs of features and find that complex interactions occur in a large portion of the pairs. The distribution of $W$ and $W^r$ are shown in Figure~\ref{fig:cacu_feature_heatmap} as a heatmap for the 7626 possible feature pairs.
    Figure~\ref{fig:cacu_feature_dist} presents the interactions between pairs of features for all four datasets.

    \begin{figure}[htbp]
    \vspace{-1em}
    \begin{center}
    \centerline{\includegraphics[width=1\linewidth]{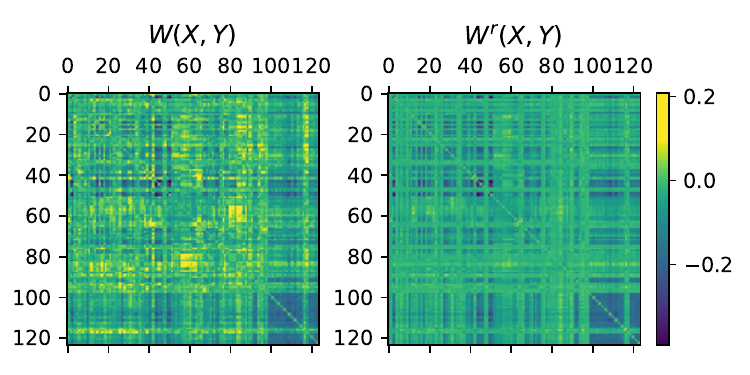}}
    \vspace{-0.5em}
    \caption{Heat-map of interactions between features in the CaCU dataset. Values are given as percentages of the original variance of $T$. Notice the relative absence of positive values on the figure on the right side, which represent complex interactions.}\label{fig:cacu_feature_heatmap}
    \end{center}
    \vspace{-0.5em}
    \end{figure}
    
    \begin{figure*}[t]
    \begin{center}
    \centerline{\includegraphics[width=\linewidth]{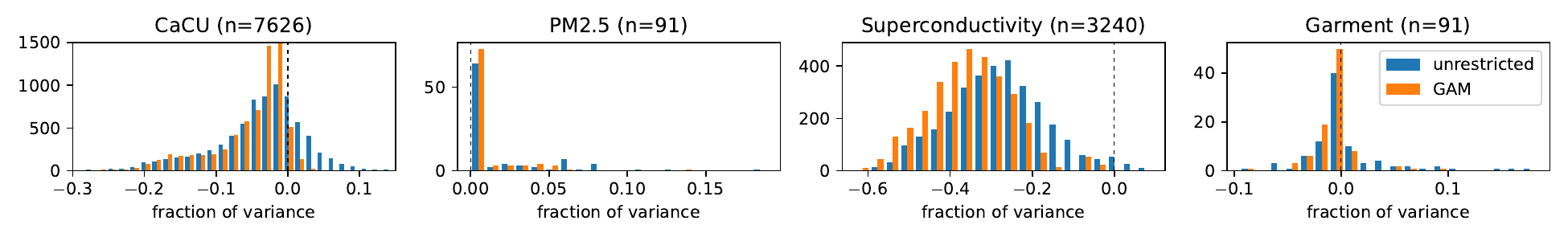}}
    \vspace{-1em}
    \caption{Histogram of interactions between pairs of features on four datasets ($n$=number of pairs). Negative values represent redundant information.}\label{fig:cacu_feature_dist}
    \end{center}
    \vspace{-1.5em}
    \end{figure*}
    
    When using the restricted models, the large majority of such interactions disappear, see again Figures~\ref{fig:cacu_feature_heatmap} and~\ref{fig:cacu_feature_dist}. The remaining positive interactions can be attributed to the small size of the dataset in both the CaCU and the Garment datasets: due to the low number of samples, measured MSE and training accuracy have relatively large variance themselves. However, in the PM2.5 dataset, some larger positive interactions remain despite the size of the dataset.
    \begin{figure}[htbp]
    \vspace{-0.5em}
    \begin{center}
    \centerline{\includegraphics[width=1\linewidth]{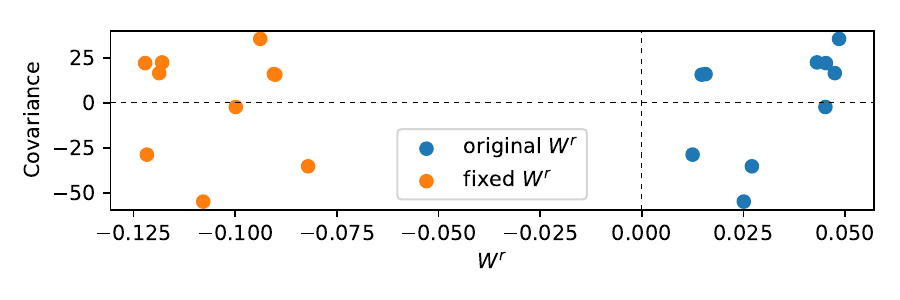}}
    \vspace{-1.5em}
    \caption{Initial and fixed covariance values in the PM2.5 dataset.}\label{fig:pm_cov}
    \vspace{-1em}
    \end{center}
    \end{figure}
    Investigating these positive interaction values, we find that many of them are \emph{not} accompanied by negative covariances between parts of the prediction model, see Figure~\ref{fig:pm_cov}, which should be impossible as per Theorem~\ref{st:redundancy}. Upon further investigation, we find that these anomalies are caused by the modeling process occasionally training sub-optimal models on the PM2.5 dataset when only a single variable is used, leading to incorrect $W^r$ values. When optionally using the single-variable function $F_Y$ of the two variable GAM models instead of individually trained single-variable models, the anomalies disappear, see Figure~\ref{fig:pm_cov}. For \ref{eq2}, overall, we do not detect examples similar to Example~\ref{ex:deficient}.
    
\begin{table*}[htbp]
      \centering
  \caption{Contributions to variance reduction for causal-attributions in the mobile telecommunications dataset}
  \label{tab:attributions}
    
    \begin{tabular}{|l|c|c|c|c|c|c|c|c|c|c|c|}
    \hline
    & $\phi_{1}$ & $\phi_{2}$ & $\phi_{3}$ & $\phi_{4}$ & $\phi_{5}$ & $\phi_{6}$ & $\phi_{7}$ & $\phi_{8}$ & $\phi_{9}$ & $\phi_{10}$ & $\phi_{\mathcal{I}}$ \\
    \hline
    \textbf{Contributions to $\sigma^2$ with GAMs}&-14.10 & 0.00 & -0.84 & -15.08 & -3.15 & -5.16 & -0.43 & -3.47 & -1.78 & -5.26 & -2.80 \\
    \hline
    \textbf{Contributions to $\sigma^2$ with unrestricted models}&-14.10 & 0.00 & -1.59 & -15.38 & -4.08 & -5.55 & -0.53 & -3.76 & -2.10 & -4.99 & 0.00 \\
    \hline
    \end{tabular}
\end{table*}

\begin{figure}[htbp]
\vspace{-1em}
\begin{center}
\centerline{\includegraphics[width=0.9\columnwidth]{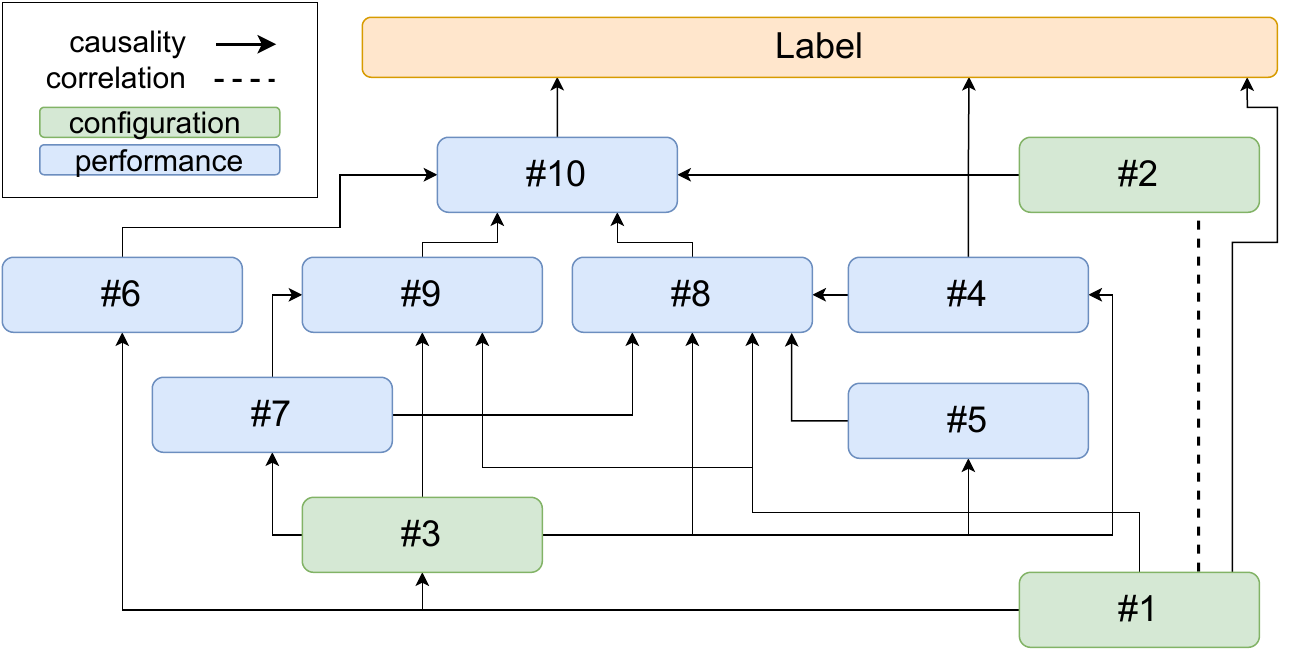}}
\vspace{-0.5em}
\caption{Causal relations between different feature groups in the mobile telecommunications dataset. See Appendix~\ref{apx:data} for variable names.}\label{fig:causality_dag}
\end{center}
\vspace{-1.0em}
\end{figure}
\subsection{Interactions between pairs of feature groups}
    \begin{figure}[htbp]
    \vspace{-0.5em}
    \begin{center}
    \centerline{\includegraphics[width=\linewidth]{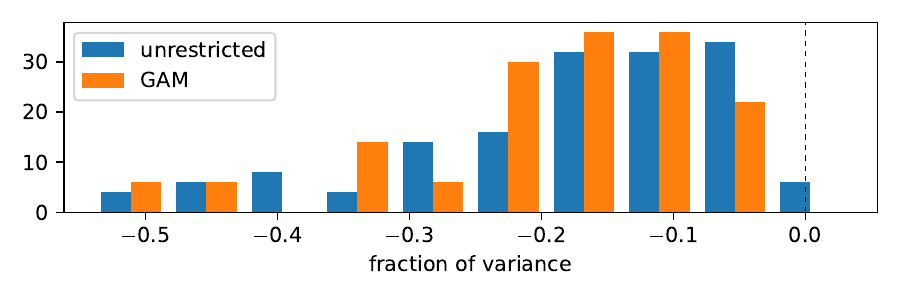}}
    \vspace{-1em}
    \caption{Histogram of interactions of pairs of semantic groups in the CaCU dataset. Negative values represent redundant information.}\label{fig:cacu_feature_dist_groups}
    \end{center}
    \vspace{-1em}
    \end{figure}
    With larger datasets, it often makes sense to group variables before defining causal relationships. In such cases, we can compute the attribution for the groups instead of single variables while allowing complex interactions within each group. To test the behavior of interactions, we group both semantically and at random. For example, we use a manually compiled semantic grouping of CaCU variables categorizing the 124 features into 13 semantic groups presented in Appendix~\ref{apx:data}.
    
    The interactions observed between pairs of variables or groups are always composed of redundant (negative) and complex (positive) interactions. A negative sum can still include positive interaction terms. In a restricted model, we expect no positive interactions and negative interactions to change to even more negative. We display the distribution of interactions between groups
    on the CaCU dataset in Figure~\ref{fig:cacu_feature_dist_groups}.

    \begin{table}[htbp]
      \centering
    \caption{The unexplained fraction of variance with different groupings of variables using GAMs vs. the unrestricted baseline. 
    }
    \label{tab:group_predictive_powers}
    \begin{tabular}{|l|r|r|}
    \hline
    \textbf{Grouping} & \textbf{\# of groups} & Remaining variance \%\\
    \hline
    \multicolumn{3}{|l|}{\textbf{Communities and Crime Unnormalized}}\\
    \hline
    Features as groups&124&0.4062\\
    \hline
    Random groups of size 6&20&0.3740\\
    \hline
    Semantic groups&13&0.3856\\
    \hline
    Unrestricted model&1&0.3803\\
    \hline
    \multicolumn{3}{|l|}{\textbf{PM2.5 Data of Five Chinese Cities}}\\
    \hline
    Features as groups&14&0.8327\\
    \hline
    Random groups of size 3&4&0.7336\\
    \hline
    Unrestricted model&1&0.6251\\
    \hline
    \multicolumn{3}{|l|}{\textbf{Superconductivity}}\\
    \hline
    Features as groups&81&0.1775\\
    \hline
    Random groups of size 6&13&0.1193\\
    \hline
    Unrestricted model&1&0.0977\\
    \hline
    \multicolumn{3}{|l|}{\textbf{Productivity Prediction of Garment Employees}}\\
    \hline
    Features as groups&14&0.6151\\
    \hline
    Random groups of size 3&4&0.5761\\
    \hline
    Unrestricted model&1&0.5533\\
    \hline
    \multicolumn{3}{|l|}{\textbf{Mobile Telecommunications}}\\
    \hline
    Semantic groups&10&0.4722\\
    \hline
    Unrestricted model&1&0.4422\\
    \hline
    \end{tabular}
    \end{table}

    For \ref{eq3}, Table~\ref{tab:group_predictive_powers} presents the predictive power of GAMs with different groupings of variables along with unrestricted models. We observe that a varying amount of predictive power is based on complex interactions, depending on the dataset. However, in the two examples with semantic groups, the performance decrease in the restricted model is very low, especially compared to using a GAM over individual features.

\subsection{Causality-aware attributions}

We test causality-aware contributions on the Telco dataset based on a causality graph of the variables established with the help of domain experts, see Figure~\ref{fig:causality_dag} and Appendix~\ref{apx:data}. Each topological ordering of the causality graph corresponds to a permutation $\pi$ for ASV. There are 1134 possible topological orderings of the causality graph with 62 distinct starting subsets of feature groups. In practice, this means training 62 models, which is feasible in our case.

Contributions to variance reduction are presented in Table~\ref{tab:attributions} for ASV with and without the GAM restrictions. 
Compared to the overall variance reduction from  $93.35$ to $41.28$, both the amount of complex interactions ($\phi_\mathcal{I}=-2.8$) and the  difference between the two kinds of models are modest, which indicates a correct causal order and lack of confounders. However, we do observe relatively large changes in some of the contributions, indicating that complex interactions indeed occur. From the small but substantial differences, domain experts can identify unexpected interactions among components at different levels of the causal graph.

\section{Conclusions}\label{sec:conclusions}
In this work, we investigated Asymmetric Shapley Values (ASV) for root-cause analysis.
We formalized ASV feature attributions in terms of variance reduction and presented examples where ASV produces counter-intuitive results. We proposed using ASV in conjunction with generalized additive models (GAMs), and proved multiple results about the joint behavior, with Theorem~\ref{st:redundancy} giving a sufficient condition for correct behavior.

We conclude that although ASV generally performs well, it may incorrectly assign contributions in various situations involving complex relationships in real-world datasets. By utilizing GAMs, most of these problematic cases can be mitigated; however, it is still possible for issues to arise under unusual circumstances. Therefore, we recommend considering the potential occurrence of the problematic cases outlined in this paper  when applying ASV. Nevertheless, when such cases are absent, ASV remains a highly effective approach for assigning contributions with known causal relationships.

\bibliographystyle{IEEEtran}
\bibliography{paper}

\section{Appendix}\label{sec:appendix}
\subsection[Existence of the restricted expectation]{Existence of $E^r$}\label{apx:existence}
As described in~\cite{brockwell2009time}, if
\begin{equation*}
    M = \{ g(Y) \mid g \text{ is a measurable  function and }\operatorname{E}(g(Y)^2)\!<\!\infty \},
\end{equation*}
then $M$ is a closed subspace of $L^2(\Omega)$. Let us define
\begin{equation}
\begin{split}
    & K = \Bigg\{ f(X) \,\Big|\, f(X) = \sum_{i=1}^{n} f_i(X_i);
    \mbox{$f$, $f_i$ are} \\ & 
    \mbox{ measurable and $E(f_i(X_i)^2) <\infty$, $i=1,2,\ldots,n$} \Bigg\}.  \end{split}
\end{equation}
Then $K\subseteq \overline{K}\subseteq M$, where $\overline{K}$ is the closure of $K$.
The function $f$ such that $E^r(T\mid X)=f(X)$ always exists with $E(f(X)^2)<\infty$, if we define $E^r(T \mid X_1,\tdots,X_n)$ as
\begin{equation}
    E^r(T\mid X_1, \tdots, X_n) = \argmin_{f(X) \in \overline{K}} \sigma^2(T-f(X))\textrm{.}
\end{equation}
Notice that this does not guarantee $\exists f_i : E^r(T\mid X)= f(X)=\sum f_i(X_i)$. In Theorem~\ref{st:redundancy}, an additional assumption is made to account for this.

\subsection[Lemma about variance reduction]{Lemma about $L_T$}\label{apx:hilbert}
\begin{lemma}\label{st:l_r_t}
\begin{gather}
   L_T(X) = \sigma^2(T) - \sigma^2(T-E(T\mid X)) = \sigma^2(E(T\mid X))
\end{gather}
\end{lemma}
\begin{proof}
\begin{gather}
   L_T(X) = \sigma^2(T) - \sigma^2(T-E(T\mid X))\\
   = 2cov(T,E(T\mid X)) - \sigma^2(E(T\mid X))\\
   = 2cov(T,E(T\mid X)) - cov(E(T\mid X), E(T\mid X))\\
   = cov(T, E(T\mid X)) + cov(T-E(T\mid X), E(T\mid X))\label{eq:cancel}\\
   = cov(T, E(T\mid X)) = \sigma^2(E(T\mid X))
\end{gather}
The second term of \eqref{eq:cancel} is zero because we know that $cov(T-g(T), f(T))$ is zero for all $f$ if $g$ minimizes $\min_g \sigma^2(T-g(T))$ which follows from the Hilbert projection theorem.
\end{proof}
\subsection[Contributions in Example~\ref{ex:nonlinearity} using conditional expectation]{Contributions in Example~\ref{ex:nonlinearity} using $E$}\label{apx:example_2}
We are going to calculate $\phi_1'$ and $\phi_2'$, the variance reductions assigned to variables $X_1$ and $X_2$ using the ASV distal weighting scheme, using the permutation $(X_1, X_2)$. First,
\begin{gather}\label{eq:e_t_mid_x1}
    E(T\mid X_1)=E(4X_1^2\mid X_1) + E(8X_1X_2\mid X_2) +{}\\{}+ E(4X_2^{2}\mid X_1)
    = 4X_1^2 + 0 + E(4X^2_2) = 4X_1^2 + c.
\end{gather}
It is known that if $Z\sim N(0,s^2)$, then $\sigma^2(Z^2)=2s^4$, thus
\begin{gather}
    \phi'_1 = -L_T(X_1) = \sigma^2(T-E(T\mid X_1))-\sigma^2(T)\\= -\sigma^2(4X_1^2) = -16\,\sigma^2(X_1^2)=-32
\end{gather}
because of Lemma~\ref{st:l_r_t} and
$
    E(T\mid X_1,X_2) = T.
$
Since $T=(2X_1+2X_2)^2$ where $(2X_1+2X_2)\sim N(0, 8)$,
\begin{gather}
    \phi'_0 = \sigma^2(T)=128\textrm{ and}\\
    \phi'_0 + \phi'_1 + \phi'_2 = 128 - 32 + \phi'_2 = \sigma^2(T\mid X_1, X_2) = 0\\\textrm{therefore }
    \phi'_2 =-96\textrm{.}
\end{gather}
\subsection[Contributions in Example~\ref{ex:nonlinearity} using restricted conditional expectation]{Contributions in Example~\ref{ex:nonlinearity} using $E^r$}\label{apx:example_2_e_r}
With using $E^r$, the behavior is a bit different. We still have
    $\phi'_0 = \sigma^2(T)=128$ and 
    $E^r(T\mid X_1) = 4X_1^2$, thus
    $\phi'_1 =-32$,
however $E^r(T\mid X_1, X_2)\neq E(T\mid X_1, X_2)$. The best restricted prediction we can wish for is the sum of the individual expectations, i.e., $
    E^r(T\mid X_1,X_2) = E(T\mid X_1) + E(T\mid X_2),
$
which is the case if $cov(E(T\mid X_1), E(T\mid X_2)) = 0$. Since $E^r(T\mid X_1) = 4X_1^2+c_1$ and $E^r(T\mid X_2) = 4X_2^2+c_2$, they are indeed uncorrelated, so
$E^r(T\mid X_1,X_2) = 4X_1^2 + 4X_2^2.$
From this, $E((T-E^r(T\mid X_1, X_2))^2)=\sigma^2(8X_1X_2)=64$,
and we can determine $\phi_2$ as
\begin{gather*}
    \phi'_0 + \phi'_1 + \phi'_2 = 128 - 32 + \phi'_2 = \sigma^2(T\mid X_1, X_2) = 64,
\end{gather*}
therefore $\phi'_2 =-32$. This means that with $E^r$, both variables get equal contribution. In this case, $\phi_{\mathcal{I}}$, defined in Section~\ref{sec:rias}, gets the most contribution: $\phi_{\mathcal{I}} = -64.$

\subsection{Proof of Theorem~\ref{st:cond_indep}}\label{apx:cond_indep}
\begin{proof}
    Because of the independence of $X$ and $Y$, for any functions $f(X)$ and $g(Y)$
    \begin{gather}
        \left[\sigma^2(T)-\sigma^2(T-f(X))\right] + \left[\sigma^2(T)\!\!- \sigma^2(T-g(X))\right]\\
        \!\!\!\!=\!2cov(T,\!f(X))\!-\!\sigma^2(f(X))\!+\!2cov(T,\!g(Y))\!-\!\sigma^2(g(Y))\!\\= 2cov(T,f(X)+g(Y))-\sigma^2(f(X)+g(Y))\\
        =\left[\sigma^2(T)-\sigma^2(T-(f(X)+g(Y)))\right].
    \end{gather}
    This trivially implies $L^r_T(X) + L^r_T(Y)\leq L^r_T(X,Y)$, as the function $L^r_T(X,Y)=\sigma^2(T)-\sigma^2(T-h(X,Y))$ is maximal for functions of the form $h(X,Y)=h_X(X)+h_Y(Y)$.
    
    We can also use it to prove $L^r_T(X) + L^r_T(Y)\geq L^r_T(X,Y)$, however some care needs to be taken due to the considerations of Appendix~\ref{apx:existence}. To do this, let us take a series of measurable square-integrable $h_i,h_{i,X}(X),h_{i,Y}(Y)$ functions such that
    \begin{gather}
        L^r_T(X,Y) = \lim_{i\rightarrow \infty}\left[ \sigma^2(T)-\sigma^2(T-h_i(X,Y))\right]\\=
        \lim_{i\rightarrow \infty}\left[ \sigma^2(T)-\sigma^2(T-(h_{i,X}(X)+h_{i,Y}(Y)))\right].
    \end{gather}
    Since $L^r_T(X), L^r_T(Y)$ are also maximal for for functions of the given form, this means that for each term of the series $h_i$
    \begin{equation}
        \!\!\!\sigma^2(T)-\sigma^2(T-(h_{i,X}(X)+h_{i,Y}(Y))\!\leq\!L^r_T(X)\!+\!L^r_T(X),
    \end{equation}
    which implies that the same is also true for the limit, i.e.
    \begin{equation}
        L^r_T(X,Y)\leq L^r_T(X) + L^r_T(X).\nonumber
    \end{equation}
    \vspace{-1em}
\end{proof}

\subsection{Proof of Theorem~\ref{st:redundancy}}\label{apx:leq_proof}
    \begin{proof}
    Let us introduce the notations
    \begin{gather}
        E^r(T\mid X_1,\tdots,X_n,Y_1,\tdots,Y_n)=F_X(X)+F_Y(Y)\\=\sum f_{X_i}(X_i) + \sum f_{Y_i}(Y_i)= F\\
        \!\!E^r(T\mid X_1,\tdots,X_n)=G,\textrm{ and }
        E^r(T\mid Y_1,\tdots,Y_m)=H.
    \end{gather} Since
    $
        \sigma^2(T)-\sigma^2\left(T-A\right) = 2cov(T, A)-\sigma^2(A)\textrm{,}
    $
    thus
    \begin{gather}
        \!\!\!\!\!\!\!L^r_{\scriptscriptstyle T}(X_1,\tdots,X_n\!)\!=\!\sigma^2(T)\!-\!\sigma^2(T\!\!-\!G)\!=\!2cov(T,\!G)\!-\!\sigma^2(G)\!\!\nonumber\\
        \!\!\!\!\!\!\!L^r_{\scriptscriptstyle T}(X_1,\tdots,X_n\!)\!=\!\sigma^2(T)\!-\!\sigma^2(T\!\!-\!H)\!=\!2cov(T,\!H)\!-\!\sigma^2(H)\!\!\!\nonumber
    \end{gather}
    and
    \begin{gather}
     L^r_{\scriptscriptstyle T}(X_1,\tdots,X_n,Y_1,\tdots,Y_m) = \sigma^2(T) - \sigma^2(T - F)\\= 2cov(T, F) - \sigma^2(F)\\
     =2cov\left(T, \sum f_{X_i}(X_i)\right) - \sigma^2\left(\sum f_{X_i}(X_i)\right)+{}
     \\
     {}+ 2cov\left(T, \sum f_{Y_j}(Y_j)\right) - \sigma^2\left(\sum f_{Y_j}(Y_j)\right) - {} \\
     {} - 2cov\left(\sum f_{X_i}(X_i), \sum f_{Y_j}(Y_j)\right).
    \end{gather}
    Thus
    \begin{gather}
        W^r_{\scriptscriptstyle T}(X_1,\tdots,X_n;Y_1,\tdots,Y_m)\\
     =\left(2cov\left(T, \sum f_{X_i}(X_i)\right) - \sigma^2\left(\sum f_{X_i}(X_i)\right)\right) -{}\\{}-(2cov(T, G)-\sigma^2(G))\label{eq:minimal1}+{}\\
     {}+
     \left(2cov\left(T, \sum f_{Y_i}(Y_i)\right) - \sigma^2\left(\sum f_{Y_i}(Y_i)\right)\right) -{}\\{}-(2cov(T, H)-\sigma^2(H))\label{eq:minimal2}+{}\\
     {}+
     \left(-2cov\left(\sum f_{X_i}(X_i), \sum f_{Y_i}(Y_i)\right)\right).
    \end{gather}
    Since~\eqref{eq:minimal1} and~\eqref{eq:minimal2} are maximal for such expressions, i.e.,
    \begin{gather}
        2cov\left(T, \sum f_{X_i}(X_i)\right) - \sigma^2\left(\sum f_{X_i}(X_i)\right)\\\leq 2cov(T, G)-\sigma^2(G), \textrm{thus}\\
        W^r_{\scriptscriptstyle T}(X_1,\tdots,X_n;Y_1,\tdots,Y_m) + c\\= -2cov\left(\sum f_{X_i}(X_i), \sum f_{Y_i}(Y_i)\right)
    \end{gather}
    for some $c\geq 0$, which proves the proposition. 
    \end{proof}
\subsection{Datasets and Variables}\label{apx:data}
In the Communities and Crime Unnormalized dataset, we define the following feature groups: Race, Age, Income, Race/Income, Education, Family, Immigration, Housing, Homelessness, Native, Police, Race/Police, Land/Population.

In the Telecommunications dataset, the feature groups indicated by $\phi_1, \tdots, \phi_{10}$ are, respectively: Spectrum (1); Antennas, MIMO and modulations (2); Dimensioning (3); Cell load (4); Neighbor cell load (5); UE capability distribution (6); TA distribution (7); Interference (8); Path loss (9); Channel quality (10) and the label is Downlink Throughput. Causal relations are displayed in Figure~\ref{fig:causality_dag}.

For full description of the groups in both the Telco and CaCu datasets, please refer to the source repository.

\end{document}